\documentclass{article}

\usepackage[english]{babel}
\usepackage[letterpaper,top=2cm,bottom=2cm,left=4cm,right=4cm,marginparwidth=1.75cm]{geometry}

\usepackage{amsmath, amsfonts}
\usepackage{graphicx}
\usepackage[colorlinks=true, allcolors=blue]{hyperref}
\usepackage{xcolor}
\definecolor{lightblue}{RGB}{70, 180, 240}
\usepackage{float}
\usepackage{algorithm}
\usepackage{algpseudocode}
\usepackage{svg}
\usepackage{lineno}
\usepackage{subcaption}

\usepackage[theorems]{macros}

\usepackage[style=ieee, citestyle=numeric]{biblatex}
\usepackage{csquotes}
\usepackage{pifont}
\newcommand{\cmark}{\ding{51}}
\newcommand{\xmark}{\ding{55}}

\newcommand{\filip}{Filip Naudot\thanks{\texttt{filipn@cs.umu.se}} \\ Umeå University}
\newcommand{\tobias}{Tobias Sundqvist\thanks{\texttt{tobias.sundqvist@tietoevry.com}} \\ Tietoevry}
\newcommand{\timotheus}{Timotheus Kampik\thanks{\texttt{tkampik@cs.umu.se}} \\ Umeå University \& SAP}

\title{llmSHAP: A Principled Approach to LLM Explainability}

\author{
    \filip
    \and
    \tobias
    \and
    \timotheus
}
\date{} % Do not display date.

\begin{document}
\maketitle
%%%%%%%%%%%%%%%%%%%%%%%%%%%%%%%%%%%%%%%%%%%%%%%%%%%%%%%%%%

%%%%%%% ABSTRACT %%%%%%%
\begin{abstract}
Feature attribution methods help make machine learning-based inference \emph{explainable} by determining how much one or several features have contributed to a model's output.
A particularly popular attribution method is based on the Shapley value from cooperative game theory, a measure that guarantees the satisfaction of several desirable principles, assuming deterministic inference.
We apply the Shapley value to feature attribution in large language model (LLM)-based decision support systems, where inference is, by design, stochastic (\emph{non-deterministic}).
We then demonstrate when we can and cannot guarantee Shapley value principle satisfaction across different implementation variants applied to LLM-based decision support, and analyze how the stochastic nature of LLMs affects these guarantees.
We also highlight trade-offs between explainable inference speed, agreement with exact Shapley value attributions, and principle attainment.
\end{abstract}

%%%%%%% PAPER %%%%%%%
% \linenumbers
%%%%%%%%%%%%%%%%%%%%%%%%%%%%%%%%
\section{Introduction}
\label{section:introduction}
%%%%%%%%%%%%%%%%%%%%%%%%%%%%%%%%
The rise of data-driven algorithms and, most notably, applications of deep learning has led to concerns about a lack of thorough human  oversight of socially important decisions that are either delegated in their entirety to machines, or made by humans based on machine recommendations.
Explainable AI (XAI) approaches attempt to mitigate these concerns by helping (typically human) users understand how and why algorithms produce specific outputs~\cite{oneil2016weapons}.
An important class of XAI methods focuses on providing explanations of black-box classifiers that attribute classification outcomes (which one may consider \emph{decisions} or \emph{decision recommendations}) to input characteristics (feature values) \cite{DBLP:conf/kdd/Ribeiro0G16,lundberg2017SHAP}.
Such \emph{feature attribution} methods can be considered meta-reasoning functions that approximate classifier behavior with the objective of providing users a reasonably faithful intuition of behavioral fundamentals.
One of the most prominent feature attribution methods is \emph{SHAP}, which is based on the Shapley value in cooperative game theory that quantifies players' (feature values') contributions to coalition utility (classification outcomes) \cite{shapley1953}.
Feature attribution methods have, in general, limitations: notably, they are necessarily approximations, and as purely technical tools, they cannot fully consider crucial nuances of the socio-technical systems they are embedded in~\cite{DBLP:journals/ai/Miller19}; for example, the visualizations provided out-of-the-box by feature attribution software libraries may be difficult to interpret~\cite{DBLP:conf/chi/KaurNJCWV20}.
Still, Shapley value-based approaches can be considered a reasonable choice for facilitating black-box explainability, notably because (i) they are based on well-established and intuitive mathematical principles of the Shapley value and (ii) there is at least some evidence of their potential usefulness, also relative to alternative approaches~\cite{DBLP:conf/chi/KaurNJCWV20}.
However, the Shapley value cannot straight-forwardly be applied to inference from Large Language Models (LLMs), which power many of the currently emerging AI applications. The reasons for this are two-fold:
\begin{enumerate}
  \item While the Shapley value assumes deterministic inference function behavior, LLM behavior tends to be deliberately stochastic in most application scenarios. This means the satisfaction of Shapley value principles is no longer guaranteed.
  \item Computing Shapley values is computationally expensive and so is LLM inference.
\end{enumerate}
In this paper, we address these challenges by providing an overview of fundamental trade-offs that need to be made when implementing Shapley value-based approaches to LLM-based decision support explainability.
These trade-offs are grounded in the theoretical foundations of the Shapley value and supported by empirical results that illustrate how design choices may play out in practice.
Our overview can help researchers to study ``downstream'' effects of different design choices, e.g., in specific applications or human-computer interaction lab studies.
Practitioners may use our overview, alongside our open-source reference implementation (available at \url{https://github.com/filipnaudot/llmSHAP}), to make informed decisions when implementing Shapley values to facilitate explainability of real-world systems.
As prerequisites, we first explain the Shapley value, its principles, and its application to XAI (Section~\ref{section:shapley-value}).
We then introduce new foundations of Shapley values for LLMs (llmSHAP), assuming stochastic inference (Section~\ref{section:methodology}).
These foundations reflect different implementation approaches for llmSHAP variants.
Given these foundations, we make theoretical observations about principles (Section~\ref{section:axiomatic-analysis}) and computational complexity (Section~\ref{section:complexity-analysis}), and show empirically how these differences may play out in practice (Section~\ref{section:empirical-evaluation}).
Finally, we provide a brief discussion of llmSHAP in the light of related work (Section~\ref{section:related-work}) before we conclude by summarizing key design decisions that are required when utilizing the Shapley value for LLM explainability (Section~\ref{section:conclusion}).

%%%%%%%%%%%%%%%%%%%%%%%%%%%%%%%%
\section{The Shapley Value}
\label{section:shapley-value}
%%%%%%%%%%%%%%%%%%%%%%%%%%%%%%%%
The Shapley value (introduced in \cite{shapley1953}) is a game theoretic approach for distributing the total payoff from a group of players to all individual players in the group. This is achieved by averaging a feature's marginal contribution across all possible feature coalitions. For every coalition $S \subseteq X$ (where $X$ denotes the set of all features) not containing $x \in X$, the payoff \emph{with} and \emph{without} $x$ is compared, given the payoff function $v: 2^X \rightarrow \mathbb{R}$. The permutation-weighted average of these differences is then $x$'s Shapley value. Formally, the Shapley value is defined as:
\begin{equation}\label{eq:original_shapley_value}
    \begin{aligned}
        \phi_{i}(v) =
         \sum_{S \subseteq X \setminus \{x_{i}\}} \frac{|S|! \cdot (|X|-|S|-1)!}{|X|!} \left[ v\bigl(S \cup \{x_{i}\}\bigr) - v\bigl(S\bigr) \right].
    \end{aligned}
\end{equation}
\autoref{fig:shapley-illustration} provides a simplified illustration of this process.
\begin{figure}[!ht]
    \centering
    \includegraphics[width=0.5\linewidth]{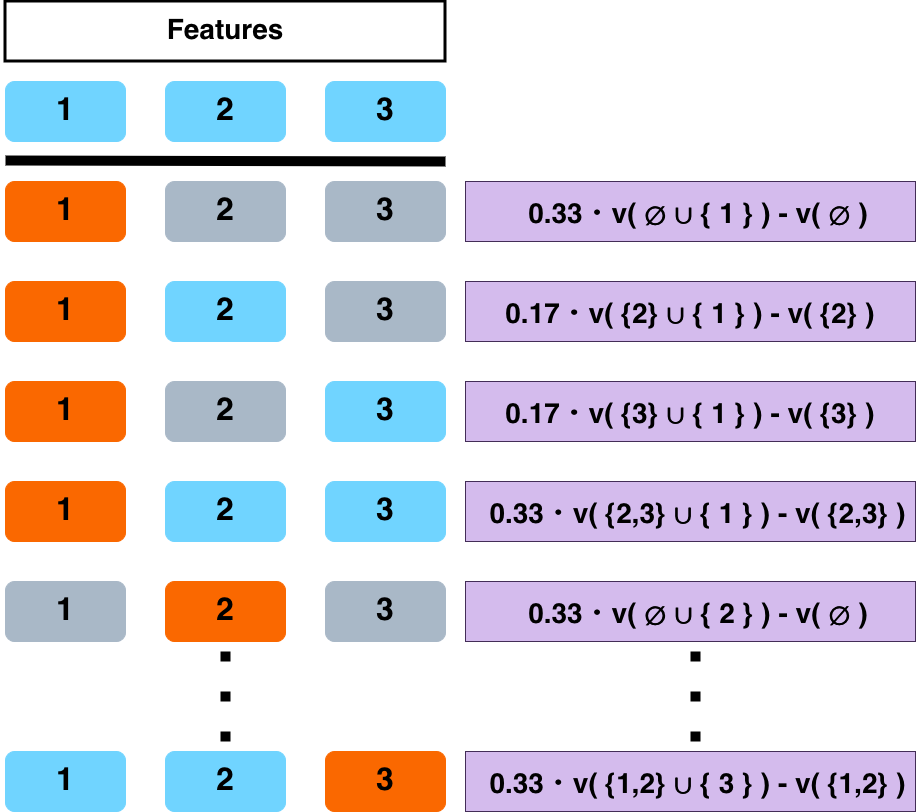}
    \caption{Illustration of Shapley value computation for each feature $x_i \in X$, with the feature of interest in \emph{\textcolor{orange}{orange}}, coalition features ($S$) in \emph{\textcolor{lightblue}{blue}}, and excluded features in \emph{gray}. For each coalition $S \subseteq X \setminus \{x_i\}$ (blue), subtract $v(S)$ from $v(S \cup \{x_i\})$. Weigh each marginal contribution by $\frac{|S|!(|X|-|S|-1)!}{|X|!}$, and sum to obtain $\phi_i(v)$ in \autoref{eq:original_shapley_value}.}
    \label{fig:shapley-illustration}
\end{figure}

The Shapley value uniquely satisfies a set of axioms.
Next, we state these axioms, which we will use as guiding desiderata later in our analyses.

\begin{axiom}[Efficiency]\label{axiom:efficiency}
    The total of all players Shapley values equals the value of the grand coalition, ensuring the entire gain is allocated among all players. Formally:
    \[
        \sum_{i \in X} \phi_i(v) = v(X).
    \]
\end{axiom}

Having fixed the total amount to be shared among players, we next present the requirement that players who are indistinguishable in their contributions be treated the same.

\begin{axiom}[Symmetry]\label{axiom:symmetry}
    Players that contribute equally to all coalitions receive equal attribution. Formally, if $v(S \cup \{x_{i}\}) = v(S \cup \{x_{j}\})$ for all $S \subseteq X \setminus \{x_{i}, x_{j}\}$, then $\phi_i(v) = \phi_j(v)$.
\end{axiom}

After accounting for equal contributions, we finally state the minimal fairness condition that a player who adds nothing should receive nothing.

\begin{axiom}[Null player (dummy)]\label{axiom:dummy}
    A player that does not contribute to any coalition receives a Shapley value of zero. Formally, if $v(S \cup \{x_{i}\}) = v(S)$ for all $S \subseteq X \setminus \{x_{i}\}$, then $\phi_i(v) = 0$.
\end{axiom}

Because we never form linear combinations of games, \emph{law of aggregation} (Axiom 3 in \cite{shapley1953}), otherwise known as \emph{linearity}, has no operational role in this work.

%%%%%%%%%%%%%%%%%%%%%%%%%%%%%%%%
\section{Shapley Values for LLMs}
\label{section:methodology}
%%%%%%%%%%%%%%%%%%%%%%%%%%%%%%%%
LLM decoding is typically stochastic: outputs are sampled from a distribution rather than chosen greedily; \emph{temperature} rescales that distribution and \emph{top-p} sampling restricts sampling to the smallest set of tokens whose cumulative probability exceeds the threshold \emph{p}. This matters because some Shapley value axioms, that hold in deterministic environments, can fail under stochastic redraws.
In this work, we focus on single-turn LLM calls where the input is a subset of the prompt (e.g., tokens/sentences/concepts) and the model's output is mapped to a scalar \emph{payoff}.

\paragraph{Setup and notation.}
In the case of LLMs, we do not have a value function $v: 2^X \rightarrow \mathbb{R}$; instead, we draw, given $S \subseteq X$, $n$ samples $d_1, ..., d_n$ from a distribution $\mathcal{D}(S)$, such that for $1 \leq i \leq n$, $d_i \in \mathbb{R}$, and aggregate them, yielding a real number. We denote this process by $h_{n}(\cdot)$, where $n$ denotes the number of samples we draw.
We may drop the subscript $n$ when the number of draws is left unspecified.
Here, we assume that the aggregation function is the average.
In what follows, we write our attribution functions as $\shap(h, X, x)$; when $h$ and $X$ are arbitrary or clear from context, we abbreviate to $\shap(x)$.

\paragraph{LLM-based Inference.}
We assume that we draw inferences from an LLM by prompting it one or several times and potentially post-processing the output.
For example, one can compute the sentiment of the output and use that score as the coalition's payoff, 
or compare each coalition's output to a reference output from the full feature set (e.g., cosine similarity between embeddings) 
and use that similarity as the basis for the payoff. Our experiments use the latter.
Drawing this inference amounts to sampling from the distribution $\mathcal{D}(S)$ as described above.

One of the most straight-forward approaches to feature attribution is to consider what happens when each feature is removed individually. 
This \emph{leave-one-out} strategy evaluates the effect of ablating a single feature at a time, while keeping all others fixed. The resulting feature attribution correspond to counterfactual reasoning: it measures how much the model's output changes when a particular feature is absent.
%
% Counterfactual.
%
Formally, the counterfactual attribution for a feature $x \in X$ is defined as:
\begin{equation}\label{equation:counterfactual}
    \shapC(h, X, x) = h(X) - h(X \setminus \{x\}).
\end{equation}

This simple approach illustrates the intuition behind feature ablation. However, this method is limited in scope, motivating a more principled extension through the Shapley value framework. To formalize this process, we first introduce an inference wrapper in Algorithm~\ref{alg:inference} that standardizes LLM calls and optionally caches results, ensuring consistent coalition evaluations.
\begin{algorithm}[H]
    \caption{Inference wrapper $\ifunc$ with optional caching}
    \label{alg:inference}
    \begin{algorithmic}[1]
        \Require Coalition $S$; inference function $h$;
            boolean \textit{use\_cache} \emph{(default: \textit{False})};
            order-invariant cache $\cache$ \emph{(default: $\{\}$)}
        \Ensure Scalar payoff $R \in \mathbb{R}$; if \textit{use\_cache} then on return $\cache[S]=R$
        \Function{$\ifunc$}{$S$, \textit{use\_cache}, $\cache$}
            \If{\textit{use\_cache}} \label{line:start-cache-lines}
                \If{$S \in \cache$}
                    \State $R \gets \textit{cache}[S]$ \Comment{Retrieve inference result from cache}
                    \State \Return $R$
                \Else
                    \State $R \gets h(S)$ \Comment{Inference result for features in $S$} \label{line:inference-call}
                    \State $\cache[S] \gets R$ \Comment{Cache inference result}
                    \State \Return $R$
                \EndIf \label{line:end-cache-lines}
            \Else
                \State \Return h(S)
            \EndIf
        \EndFunction
    \end{algorithmic}
\end{algorithm}
The Shapley value attribution method $\shapS$ follows directly from \autoref{eq:original_shapley_value}, with $v$ replaced by $\ifunc$ (from Algorithm~\ref{alg:inference}), and is described in Algorithm~\ref{algo:shapley}.
%
% Shapley value.
%
\begin{algorithm}[H]
    \caption{$\shapS$: Shapley value feature attribution for all $x \in X$}
    \label{algo:shapley}
    \begin{algorithmic}[1]
        \Require Feature set $X$, similarity function $\kappa$, inference wrapper $\ifunc$
        \Ensure Mapping $x \mapsto \shapS(h,X,x)$ for all $x \in X$
        \State $\textit{attribution} \gets \{\}$ \Comment{Initialize empty attribution map}
        \For{\textbf{each} $x \in X$}
            \State $value \gets 0$ \Comment{Accumulator for $\shapS(h,X,x)$}
            \For{\textbf{each} $S \subseteq X \setminus \{x\}$} \label{line:for-each-feature}
                \State $R_{with} \gets \ifunc(S \cup \{x\})$ \Comment{Inference result for features in $S \cup \{x\}$}
                \State $R_{without} \gets \ifunc(S)$ \Comment{Inference result for features in $S$}
                \State $weight \gets \frac{|S|!\,(|X|-|S|-1)!}{|X|!}$
                \State $value \gets value + weight \cdot \big(R_{with} - R_{without}\big)$
            \EndFor
        \State $\textit{attribution}[x] \gets value$
        \EndFor
        \State \Return $\textit{attribution}$
    \end{algorithmic}
\end{algorithm}
%
% Cache Shapley value.
%
As repeated LLM inference calls can become a computational bottleneck, caching LLM responses offers a practical way to reduce calls while still aligning with the theoretical definition. We achieve this by introducing an order-invariant cache $\cache$ and replacing $\ifunc(\cdot)$ with $\ifunc(\cdot, \textit{use\_cache=True}, \cache)$ in Algorithm~\ref{algo:shapley} and denote the resulting approach by $\shapCS$.

However, even with caching, the number of coalitions still grows exponentially with the number of features. To address this, we define a sliding-window approach $\shapSW$ that constrains the number of coalitions. Intuitively, the sliding-window approach can be seen as progressing between counterfactual and standard Shapley attributions. When the window size is 1, it reduces to the counterfactual case where each feature is ablated individually, while a window size equal to the number of input features recovers the standard Shapley value computation over the full feature coalition space.

Given a set of features $X = \{x_1, x_2, \ldots, x_n\}$, and a window $W_i = \{x_i, x_{i+1}, \dots, x_{i+w-1}\}$, that slides over all features in $X$. For $S \subseteq W_i$, let $S' = S \cup X \setminus W_i$ be the set of features containing the current coalition $S$ and all features outside of the window $W_i$. We can now compute the local Shapley value for the features within $W_i$ at each step $i$ as described in Algorithm~\ref{algo:sliding-window}. 

%
% Sliding Winndow Shapley value.
%
\begin{algorithm}[H]
    \caption{$\shapSW$: sliding window-based Shapley value feature attribution for all $x \in X$}
    \label{algo:sliding-window}
    \begin{algorithmic}[1]
        \Require Feature set $X$, window size $w$, inference wrapper $\ifunc$
        \Ensure Mapping $x \mapsto \phi^{\textit{SW}}_{w}(h,X,x)$ for all $x \in X$
        \State $\textit{attribution} \gets \{\}$ \Comment{Initialize empty attribution map}
        \State $\textit{count} \gets \{\}$ \Comment{Initialize accumulator for averaging}
        
        \For{\textbf{each} $x \in X$} \State $\textit{attribution}[x] \gets 0$;\; $\textit{count}[x] \gets 0$ \EndFor
        
        \State $\textit{num\_windows} \gets n-w+1$
        
        \For{$i = 1$ \textbf{to} \textit{num\_windows}} \Comment{Stride $1$ sliding window}
            \State $W \gets \{x_i,\dots,x_{i+w-1}\}$;\ \Comment{Initialize the current window}
            \State $X_{out} \gets X \setminus W$ \Comment{Store outside-of-window features} \label{line:init-fixed-features}
            \For{\textbf{each} $x \in W$}
                \State $local \gets 0$
                \For{\textbf{each} $S \subseteq W \setminus \{x\}$} \label{line:local-coalitions}
                    \State $S' \gets S \cup X_{out}$ \Comment{Always include outside features} \label{line:add-fixed-features}
                    \State $R_{\text{with}} \gets\ifunc(S' \cup \{x\})$ \Comment{Inference result for features in $S' \cup \{x\}$}
                    \State $R_{\text{without}} \gets\ifunc(S')$         \Comment{Inference result for features in $S'$}
                    \State $weight \gets \frac{|S|!\, (|W|-|S|-1)!}{|W|!}$
                    \State $local \gets local + weight \cdot \big(R_{\text{with}} - R_{\text{without}}\big)$
                \EndFor
                \State $\textit{attribution}[x] \gets \textit{attribution}[x] + local$
                \State $\textit{count}[x] \gets \textit{count}[x] + 1$
            \EndFor
        \EndFor
        \For{\textbf{each} $x \in X$}
            \State $\textit{attribution}[x] \gets \textit{attribution}[x] / \textit{count}[x]$ \Comment{Average attribution}
        \EndFor
        \State \Return $\textit{attribution}$
    \end{algorithmic}
\end{algorithm}
%

%%%%%%%%%%%%%%%%%%%%%%%%%%%%%%%%
\section{Axiomatic Analysis}
\label{section:axiomatic-analysis}
%%%%%%%%%%%%%%%%%%%%%%%%%%%%%%%%
In what follows, we analyze our Shapley value-based attribution methods with respect to the classical axioms when applied to LLMs. We take the ``no-information'' baseline to be $h(\emptyset)$ (as in \cite{lundberg2017SHAP}) rather than forcing $v(\emptyset) = 0$ (as in the original paper introducing Shapley values \cite{shapley1953}).
Intuitively, one might argue that if there are no features, there can be no payoff. However, in the case of LLMs, an empty string is still a valid input. Therefore, we measure the model's behavior when none of the original features are present (an empty string) and use this as the reference point. A summary of the results is presented in Table~\ref{tab:axiomatic-compliance-summary}.
\begin{table}[ht!]
    \centering
    \renewcommand{\arraystretch}{1.3}
    \begin{tabular}{c|c|c|c}
        Method & Efficiency & Symmetry & Null player (dummy)\\
        \hline
        $\shapS$ & \xmark & \cmark & \cmark \\
        \hline
        $\shapCS$ & \cmark & \cmark &  \cmark \\
        \hline
        $\shapSW$ & \xmark & \xmark & \cmark \\
        \hline
        $\shapC$ & \xmark & \cmark & \cmark \\
    \end{tabular}
    \caption{Axiomatic compliance of attribution methods (\cmark\ and \xmark\ denotes satisfied and violated axioms, respectively).
    Here, $\shapS$ denotes the Shapley value, $\shapCS$ the cache-based Shapley value, $\shapSW$ the sliding window-based Shapley value, and $\shapC$ the counterfactual.}
    \label{tab:axiomatic-compliance-summary}
\end{table}

Under a deterministic $h$, which our cache-based method $\shapCS$ enforces by caching and reusing generated model outputs, Shapley's axioms apply verbatim to $h$, and the axioms are thus satisfied naturally for $\shapCS$.

%%%
\begin{proposition}
    $\shapCS$ satisfies the Shapley axioms efficiency \ref{axiom:efficiency}, symmetry \ref{axiom:symmetry}, and Null player (dummy) \ref{axiom:dummy}.
\end{proposition}
\begin{proof}
    Since $\shapCS$ is exactly the Shapley value computed on the characteristic function $h$, the original proofs for \emph{symmetry}, and null player (dummy) (from \cite{shapley1953}) apply with $v$ replaced by $h$. The cache mechanism in $\shapCS$ changes only how $h(S)$ is obtained, therefore the same conclusion follows. For efficiency, observe that subtracting the constant $h(\emptyset)$ from the value of all coalitions does not change any marginal contribution, therefore:
    \[
        h(X) = h(\emptyset) + \sum_{i \in X} \shapCS_i(h),
    \]
    i.e., efficiency is interpreted relative to the baseline $h(\emptyset)$ (and thus cancels out).
\end{proof}
%%%

In contrast, under the desired probabilistic decoding regime (using $\shapS$), calls to $h(S)$ are stochastic across evaluations of the same $S$. This breaks the telescoping argument behind efficiency at the sample-level.
%%%
\begin{proposition} ($\shapS$ provides no guarantees for efficiency under independent redraws).
    If each occurrence of the same coalition $S$ is evaluated by a stochastic draw from $h(S)$, then there are no guarantees that we will observe:
    \[
         \sum_{x \in X} \shapS(x) = h(X) - h(\emptyset).
    \]
\end{proposition}
\begin{proof}
    In the inner sum over marginal contributions for a coalition of the features $X$, every intermediate coalition $S$ appears once with a positive sign and once with a negative sign in front of it. These are two independent draws $A_{with}, A_{without}$ (where $A_{with} = A_{without}$) from the distribution $\mathcal{D}(S)$. Unless $h$ is deterministic, $h(A_{with}) = h(A_{without})$ does not hold in general (i.e., it happens with a probability $< 1$ as illustrated in Figure~\ref{fig:distribution-illustration}). So the two terms $h(A_{with})$ and $h(A_{without})$ do not cancel each other out, leaving extra non-zero residual terms in the sum. Hence, we get:
    \[
         \sum_{x \in X} \shapS(x) \neq h(X) - h(\emptyset),
    \]
    with a probability $> 0$.
\end{proof}
%%%
\begin{figure}[ht!]
    \centering
    \includegraphics[width=0.50\linewidth]{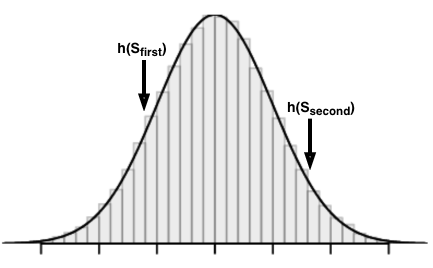}
    \caption{Illustration of how two inferences drawn by the inference function $h$ using the same feature coalition $S$ may yield different results, as the underlying sampling process is governed by a probability distribution.}
    \label{fig:distribution-illustration}
\end{figure}

%%%
\begin{proposition}
    $\shapS$ satisfies symmetry (Axiom \ref{axiom:symmetry}).
\end{proposition}
\begin{proof}
    For the Shapley value feature attribution $\shapS$ (Algorithm \ref{algo:shapley}), if for two features $x_{i}, x_{j} \in X$ we observe:
    \[
        h(S \cup \{x_{i}\}) = h(S \cup \{x_{j}\}) \quad \forall S \subseteq X \setminus \{x_{i}, x_{j}\},
    \]
    then $\shapS(x_{i})$ = $\shapS(x_{j})$.
    This is the standard Shapley symmetry axiom (Axiom \ref{axiom:symmetry}): equal marginal contribution for $x_{i}$ and $x_{j}$ across all coalitions imply equal value. Since $\shapS$ is exactly the Shapley value computed on the characteristic function $h$, the original proof (from \cite{shapley1953}) applies with $v$ replaced by $h$.
\end{proof}
Conceptually, the axiom's antecedent amounts to the assertion of an empirical observation: if the measured effect (over the function $h$) of adding $x_{i}$ equals that of adding $x_{j}$ for every coalition we evaluate, then the attribution method assigns them equal credit. We are not asserting that the condition is probable in real-world scenarios, only that if it would be observed, the stated implication is guaranteed by the Shapley value construction.
%%%

%%%
\begin{proposition}
    $\shapS$ satisfies the Shapley axiom Null player (dummy, Axiom~\ref{axiom:dummy}).
\end{proposition}
\begin{proof}
    For the Shapley value approach $\shapS$ (Algorithm~\ref{algo:shapley}), if for some feature $x \in X$ we observe:
    \[
        h(S \cup \{x\}) = h(S) \quad \forall S \in X \setminus \{x\},
    \]
    then $\shap(x) = 0$. For the Shapley \emph{dummy} axiom (Axiom \ref{axiom:dummy}), we have that if every marginal contribution of $x$ is zero, then every term in the summation for $x$ vanishes, yielding $\shap(x) = 0$. Once more, $\shapS$ is the exact Shapley value over the function $h$, so the original proof (from \cite{shapley1953}) applies.
\end{proof}
Again, the premise is conditional: if we empirically find that adding $x$ never changes $h$, then the attribution must be zero. This frames the axiom as an ``if observed, then guaranteed'' property rather than a universal claim (such as \emph{efficiency}).
%%%

We next turn to the sliding-window approximation $\shapSW$, which unlike $\shapS$ and $\shapCS$, changes the game itself.
%%%
\begin{proposition}
    $\shapSW$ violates the Shapley axioms efficiency \ref{axiom:efficiency} and symmetry \ref{axiom:symmetry}.
\end{proposition}
\begin{proof}
    Let $X = \{a,b,c,d\}$, with a window size of 2:
    $W_{1} = \{a,b\}, W_{2} = \{b,c\},$ and $W_{3} = \{c,d\}$.
    We define a deterministic game:
    \begin{itemize}
        \item $h(\{a,b\}) = 1$,
        \item $h(\{b,c\}) = 1$, 
        \item $h(\{a,b,c,d\}) = 2$, 
        \item and $h(S) = 0$ for every other coalition $S \subseteq X$.
    \end{itemize}
    The sliding window-based attribution function $\shapSW_{w=2}$ computes the local Shapely value for each feature, then averages over the windows.
    From this game we get the \emph{local} values: 
    \begin{itemize}
        \item In $W_{1} = \{a,b\}$: $a$ and $b$ each get $1/2$.
        \item In $W_{2} = \{b,c\}$: $b$ and $c$ each get $1/2$.
        \item In $W_{3} = \{c,d\}$: $c$ and $d$ each get $0$.
    \end{itemize}
    Averaging across windows that contain each feature yields:
    $\shapSW_{w=2}(a) = 1/2$, $\shapSW_{w=2}(b) = 1/2$, $\shapSW_{w=2}(c) = 1/4$, and $\shapSW_{w=2}(d) = 0$.
    
    \vspace{5pt} %%%
    \noindent \textbf{Efficiency is violated.} $h(X) = 2$, but $\sum_{x \in X} \shapSW_{w=2}(x) = 1/2 + 1/2 + 1/4 + 0 = 5/4 \neq 2$.
    
    \vspace{5pt} %%% 
    \noindent \textbf{Symmetry is violated.} In the original game, $a$, and $c$ are symmetric, yet $\shapSW_{w=2}(a) = 1/2 \neq 1/4 = \shapSW_{w=2}(c)$.
    
    \vspace{5pt} %%%
    Since deterministic games are a special case of probabilistic games; the same conclusions follow for a non-deterministic $h$.
\end{proof}
%%%

%%%
\begin{proposition}
    $\shapSW$ satisfies the Shapley axiom Null player (dummy) \ref{axiom:dummy}.
\end{proposition}
\begin{proof}
    Assume $x \in X$ does not contribute to any coalition: $h(S) = h(S \cup \{x\})$ for all $S \subseteq X \setminus \{x\}$. For any window $W$ containing $x$ and any $T \subseteq W \setminus \{x\}$ we have $T \subseteq X \setminus \{x\}$, hence $h(T) = h(T \cup \{x\})$. Thus $x$ is a null player in the restricted game on $W$, so its Shapley value in $W$ is zero (Axiom~\ref{axiom:dummy}).
    By Algorithm~\ref{algo:sliding-window}, $\shapSW(x)$ is the average of the per-window Shapley values over the windows containing $x$, hence $\shapSW(x) = 0$.
\end{proof}
%%%

Much like the sliding-window approach $\shapSW$, the counterfactual approach $\shapC$ changes the underlying game, so some standard Shapley guarantees may fail.
\begin{proposition}
    $\shapC$ violates the Shapley axiom efficiency \ref{axiom:efficiency}.
\end{proposition}
\begin{proof}
    To see that the counterfactual $\shapC$ can violate efficiency, it suffices to give a counter example where $\sum_{x \in X} \shapC(x) \neq h(X) - h(\emptyset)$.
    
    Let $X =\{a,b\}$ with $h(\emptyset) = 0$, $h(\{a\}) = h(\{b\}) = 0$, and $h(\{a,b\}) = 1$.
    Then, by Equation~\ref{equation:counterfactual}, $\shapC(a) = 1 - 0 = 1$, $\shapC(b) = 1 - 0 = 1$, $\sum_{x \in X} \shapC(x) = 2 \neq 1 = h(X) - h(\emptyset)$.
\end{proof}
%%%

Intuitively, $\shapC$ measures the effect of removing a single feature from the grand coalition. If features $x_{i}$ and $x_{j}$ behave identically, their effects match; if $x_{i}$ does not change anything, the effect is zero.
%%%
\begin{proposition}
    $\shapC$ satisfies the Shapley axioms symmetry \ref{axiom:symmetry} and Null player (dummy) \ref{axiom:dummy}.
\end{proposition}
\begin{proof}
    Since the attribution, using $\shapC$, is obtained by a single marginal contribution of feature $x \in X$ with respect to the grand coalition we get:

    \vspace{5pt} %%%
    \noindent \textbf{Symmetry.} $\shapC(x_{i}) = h(X) - h(X \setminus \{x_i\}) = h(X) - h(X \setminus \{x_j\}) = \shapC(x_{j})$, for any two symmetrical $x_{i}$ and $x_{j}$.
    
    \vspace{5pt} %%%
    \noindent \textbf{Null player (dummy).} $\shapC(x) = h(X) - h(X \setminus \{x\}) = 0$ for all $x \in X$ that contributes nothing to any coalition (which there is only one of since we only use the grand coalition).
\end{proof}
%%%

%%%%%%%%%%%%%%%%%%%%%%%%%%%%%%%%
\section{Computational Complexity Analysis}
\label{section:complexity-analysis}
%%%%%%%%%%%%%%%%%%%%%%%%%%%%%%%%
We analyze the computational costs in relation to the number of LLM inference calls of the three attribution functions $\shapS$, $\shapCS$, and $\shapSW$. Let $n = |X|$ be the number of features.
Each coalition evaluation computes $h(S)$ where the LLM query typically dominates runtime. Post-processing to combine coalition values into attribution scores is linear and negligible compared to LLM inference time.

    \vspace{5pt} %%%
    \noindent \textbf{$\shapS$.} The summation for each $x \in X$ runs over all $S \subseteq X \setminus \{x\}$ (see line~\ref{line:for-each-feature} in Algorithm~\ref{algo:shapley}). If we do not reuse coalition results across features, we evaluate both $h(S)$ and $h(S \cup \{x\})$ via LLM inference calls for every term, yielding $\bigO(2 * 2^{n-1} * n) = \bigO(2^{n}n)$.
    
    If instead one computes all unique coalition values once and reuses them, the LLM inference call count is reduced.

    \vspace{5pt} %%%
    \noindent \textbf{$\shapCS$.} Replace every LLM coalition query by a cache lookup, keyed by an order-invariant representation of $S$, and populating on miss (see lines~\ref{line:start-cache-lines} - \ref{line:end-cache-lines} in Algorithm~\ref{alg:inference}). This computes the LLM response for each coalition once (see line~\ref{line:inference-call} in Algorithm~\ref{alg:inference}), yielding $\bigO(2^{n})$.
    
    Still, the coalition count grows exponentially. To combat this, the sliding-window approach only computes the coalitions among the in-window features.
    
    \vspace{5pt} %%%
    \noindent \textbf{$\shapSW$.} For a stride of 1 and a window $W \subseteq X$ of size $w$, there are $n_{w} = n - w + 1$ windows. Within a window, $\shapSW$ compute Shapley values on the restricted game by holding $X \setminus W$ fixed (the features not in the current window, see line~\ref{line:init-fixed-features} and \ref{line:add-fixed-features} in Algorithm~\ref{algo:sliding-window}), then averaging the per-window scores for each feature, yielding $\bigO(2^{w}n_{w})$.
    
    Hence, the number of coalitions is not dependent on the number of features $n$, but only the fixed window size $w$ (see line~\ref{line:local-coalitions} in Algorithm~\ref{algo:sliding-window}).

%%%%%%%%%%%%%%%%%%%%%%%%%%%%%%%%
\section{Empirical Evaluation}
\label{section:empirical-evaluation}
%%%%%%%%%%%%%%%%%%%%%%%%%%%%%%%%
We augment our theoretical analysis with a basic empirical assessment of how heuristics-based attributions trade-off faithfulness to the Shapley value and computation time.
Specifically, we assess how closely our attribution methods ($\shapC, \shapCS, \shapSW$) match the standard Shapley value ($\shapS$), and we analyze their runtime profiles.
For the former, we compare the resulting feature-attribution vectors  to those from the standard Shapley value computation using \emph{cosine similarity}; for the latter, we measure the wall-clock time each method requires.

Our evaluation pipeline is straightforward. We use the \emph{OpenAI} API with the \emph{gpt-4.1-mini} model with a temperature of $0.2$.
For this work, investigating optimal temperature settings is out of scope.
We restrict evaluation to a single LLM because the experiments primarily serve to illustrate our theoretical results.
For each instance, we query the LLM with the full input prompt (full feature set) to obtain a \emph{base} answer.
We then form coalitions according to each method, re-query the LLM with the coalition-specific prompt, and embed both the \emph{coalition-specific} answer and the \emph{base} answer with the \emph{all-MiniLM-L6-v2} sentence transformer model, which we denote by $E$.
Intuitively, if a coalition removes information that the model actually uses, its answer should drift away from the base answer in the embedding space.
We quantify that drift by cosine similarity between the two embeddings.
Lower similarity indicates a larger semantic drift in the output.
We formalize the comparison via the similarity function $\kappa$, as also used in \cite{goldshmidt2024tokenshap, goldshmidt2025pixelshap, wang2025multishap}:
\[
    \kappa(R_{X}, R_{S}) = \frac{E(R_{X}) \cdot E(R_{S})}{||E(R_{X})|| \cdot ||E(R_{S})||},
\]
where $R_{X}$ is the model's response to the full prompt (all features in $X$ included) and $R_{S}$ is the model's response to the prompt where only the features in the coalition $S \subseteq X$ are included.

We use the \emph{disease-symptom-description} dataset from Kaggle \cite{disease-dataset}, which provides mappings between symptoms and corresponding diseases. Each record lists a set of observed symptoms associated with a particular condition. In our setup, we treat each symptom as a feature and construct LLM prompts that describe a hypothetical patient's condition. We achieve this by parsing each record into a structured prompt of the form ``A patient is showing the following symptom(s): [list of symptoms]. Based on these symptom(s), what disease or condition do you think they most likely have?'', which serves as the input for the inference and attribution procedures. We do not claim that our results support this high-risk use case of LLMs, rather, the use case merely serves as an illustrating example.

The resulting attribution vector for each method ($\shapC, \shapCS, \shapSW$) is then compared to the attribution vector of the standard Shapley value approach $\shapS$, which we use as a \emph{gold standard}, using \emph{cosine similarity}. 
Intuitively, we want faster attributions while still maintaining results that closely match those produced by the more computationally demanding standard Shapley value $\shapS$.

The resulting similarity results are presented in Figure~\ref{fig:similarity-based-results} and the time-based results are presented in Figure~\ref{fig:time-based-results}.
In Figure~\ref{fig:similarity-per-feature-count}, we observe that the cache-based method $\shapCS$ maintains the most stable similarity to the standard Shapley value across different feature counts, indicating that it yields consistent attributions regardless of input size.
In contrast, the counterfactual method $\shapC$ fluctuates notably, suggesting that its results are more sensitive to the number of features and data. 
The sliding-window approach $\shapSW_{w=3}$ yields stability levels between those of $\shapC$ and $\shapCS$, which is expected since its window size of $3$ conceptually sits between their respective coalition scopes.
In Figure~\ref{fig:time-based-results}, the results reflect the expected growth patterns of the methods: both the standard Shapley value $\shapS$ and its cache-based variant $\shapCS$ exhibit exponential scaling with the number of features, while the counterfactual $\shapC$ and sliding-window $\shapSW$ methods grow approximately linearly.
This aligns with their theoretical computational complexity, where $\shapC$ and $\shapSW$ restrict the coalition space to fixed-size subsets.
The slight drop in the sliding-window curve toward the end is likely due to network/API-level latency variability and the fact that response length can differ across prompts which may increase or decrease runtime.
\begin{figure*}[!ht]
    \centering
    \subfloat[\emph{Cosine similarity} between the attribution vectors
    averaged over all data points.\label{fig:average-similarity}]{
        \includegraphics[width=0.48\textwidth]{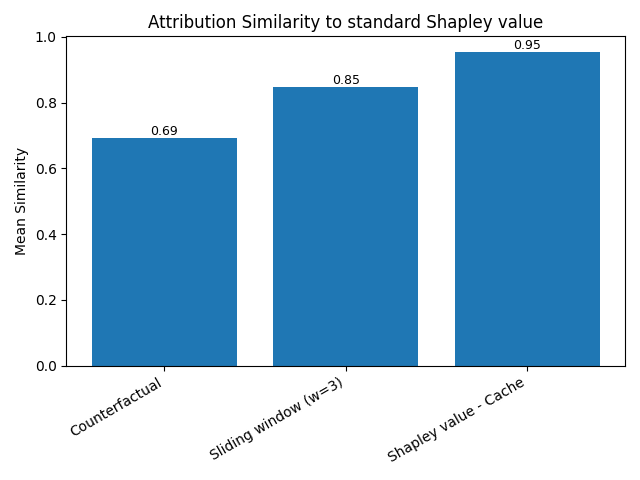}
    }
    \hfill
    \subfloat[\emph{Cosine similarity} between the attribution vectors
    averaged over the number of data points with each feature count. Dotted lines indicate results obtained under ``deterministic'' conditions (with \emph{temperature}=0 and \emph{seed}=42), 
    whereas solid lines correspond to non-deterministic conditions (with \emph{temperature}=0.2 and no fixed seed).
    \label{fig:similarity-per-feature-count}]{
        \includegraphics[width=0.48\textwidth]{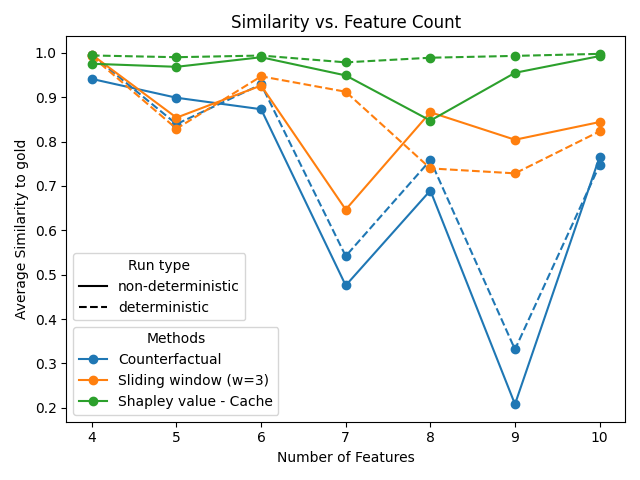}
    }
    \caption{\emph{Cosine similarity} between the attribution vectors of the
    standard Shapley value $\shapS$ (gold standard) and the counterfactual
    $\shapC$, sliding window (window size of 3) $\shapSW_{w=3}$, and the
    cached-based Shapley value $\shapCS$.}
    \label{fig:similarity-based-results}
\end{figure*}
\begin{figure*}[!ht]
    \centering
    \subfloat[Normal scale y-axis.\label{fig:normal-scale-y}]{
        \includegraphics[width=0.48\textwidth]{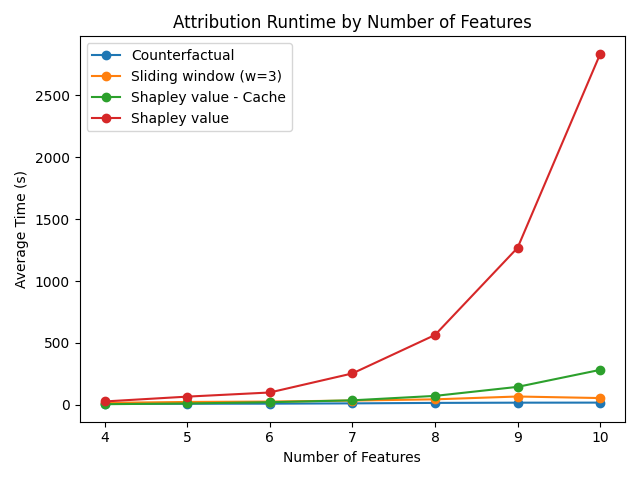}
    }
    \hfill
    \subfloat[Log scaled y-axis.\label{fig:log-scale-y}]{
        \includegraphics[width=0.48\textwidth]{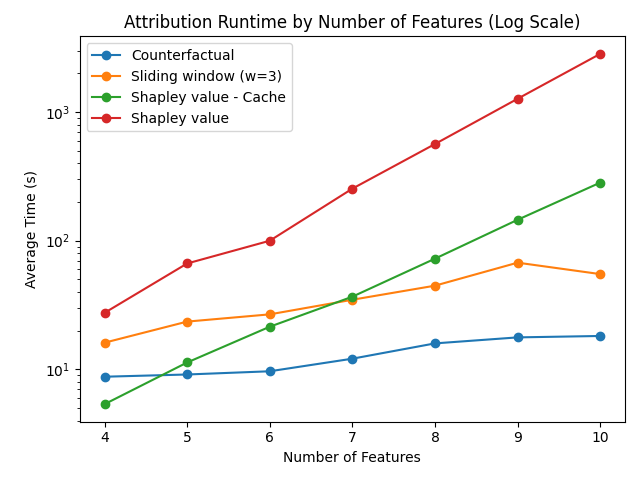}
    }
    \caption{Average runtime (seconds) for 4--10 features, based on two runs per feature count.}
    \label{fig:time-based-results}
\end{figure*}
%

%%%%%%%%%%%%%%%%%%%%%%%%%%%%%%%%
\section{Related Work}
\label{section:related-work}
%%%%%%%%%%%%%%%%%%%%%%%%%%%%%%%%
The SHAP framework \cite{lundberg2017SHAP} unifies a range of feature attribution methods under the axiomatic principles of the Shapley value, providing a theoretically grounded approach to model interpretability. 
Building on this foundation, TokenSHAP \cite{goldshmidt2024tokenshap} extends Shapley value-based explanations to LLMs by attributing importance to individual tokens through \emph{Monte Carlo} sampling, improving computational efficiency and alignment with human intuition. 
However, TokenSHAP's coalitional sampling remains unstable (fundamentally random) and the paper does not explicitly analyze how the Shapley axioms are affected under these circumstances.
In contrast, llmSHAP takes a more principled approach by formalizing how LLM-decoding stochasticity affects Shapley axioms and by proposing deterministic variants that restore axiomatic guarantees. 
Two closely related lines show how general this Shapley value view has become.
MultiSHAP \cite{wang2025multishap} applies Shapley interaction ideas to multimodal models in order to attribute not just to single features but to cross-modal pairs (image patch and text token), which is conceptually similar to llmSHAP's goal of making feature sets/concepts explainable, but it is specialized to multimodal fusion. ConceptSHAP \cite{NEURIPS2020conceptSHAP} instead tackles the problem of identifying and ranking a complete set of concepts. 
In contrast, llmSHAP assumes that such features/concepts/spans are already identified and shows how to attribute them even when inference is stochastic something TokenSHAP's token-only approach does not support.

Taken together, this highlights the contribution of this work, which is to systematically analyze and ensure Shapley principle compliance in the context of non-deterministic LLM inference.

%%%%%%%%%%%%%%%%%%%%%%%%%%%%%%%%
\section{Conclusions}
\label{section:conclusion}
%%%%%%%%%%%%%%%%%%%%%%%%%%%%%%%%
In this paper, we have conducted an analysis of formal and empirical properties of different variants of Shapley value-based attribution approaches.
From this analysis, we can derive an overview of design choices---in the form of trade-offs---that engineers need to make when applying Shapley values to LLM explainability:
\begin{itemize}
  \item Fixing the value function by caching the results of LLM inference calls (using $\shapCS$) guarantees Shapley value principle satisfaction and speeds-up computation. However, it may give users a distorted impression of inference behavior that is intuitively ``less stochastic'' than in the actual application.
  \item Sampling the LLM repeatedly for each coalition lets us estimate its expected payoff and reduce decoding noise, improving stability of the attribution. However, it does not restore efficiency under stochastic redraws. 
  Intuitively, repeated sampling smooths randomness by replacing each coalition payoff with an empirical mean over multiple redraws, so attributions target the expected payoff rather than a single draw.
  This might stabilize values but does not restore efficiency.
  \item Likewise, one must choose the appropriate coalition structure: considering all features $\shapS$, focusing on single-feature effects $\shapC$, or adopting an intermediate compromise $\shapSW$. Applying a sliding window over the features speeds up computation at the cost of (heuristic) principle satisfaction. The smaller the window size, the faster the explanations can be computed. However, smaller window sizes are not associated with proportionally worse approximation, which indicates that making trade-offs in this regard may be compelling.
\end{itemize}
We encourage future work that investigates how the above trade-offs affect the crucially important human-computer interaction aspect of LLM explainability.

Our approach centers on \emph{input feature attribution}, which inherently requires multiple inference rounds.
While this is tractable for many inference procedures, it becomes increasingly challenging for reasoning-oriented LLMs, where each inference step can be computationally expensive (since this includes the internal ``thinking'').
Still, the principles behind our framework extend naturally beyond raw input features. 
One can view the internal reasoning steps of an LLM (its \emph{chain-of-thought}) as higher level ``features'' (cf. \cite{bogdan2025thoughtanchorsllmreasoning}).
Applying attribution analysis on this level could reveal how much each internal reasoning step contributes to the final outcome, similarly to how input features contribute to the final output.

%%%%%%%%%%%%%%%%%%%%%%%%%%%%%%%%
\subsubsection*{Acknowledgments}
\label{section:acknowledgments}
%%%%%%%%%%%%%%%%%%%%%%%%%%%%%%%%
This work was partially supported by the Wallenberg AI, Autonomous Systems and Software Program (WASP) funded by the Knut and Alice Wallenberg Foundation.

%%%%%%% END %%%%%%%    
% \newpage
\printbibliography
\end{document}